\documentclass[12pt]{amsart}
\usepackage[left=15mm, right=15mm, top=10mm, bottom=15mm]{geometry}

\usepackage[english]{babel}
\usepackage[T1]{fontenc}
\usepackage[utf8]{inputenc}

\theoremstyle{plain}
    \newtheorem{theorem}{Theorem}[section]
    \newtheorem{lemma}[theorem]{Lemma}

\theoremstyle{definition}

\usepackage[
    draft = false,
    unicode = true,
    colorlinks = true,
    allcolors = blue,
    hyperfootnotes = true,
    citecolor = red
]{hyperref}
\usepackage{amsfonts,amsmath,amssymb,amscd,amsthm,url}
\usepackage[inline]{enumitem}
\usepackage{graphicx,epsf,afterpage, wrapfig}
\usepackage{soul}
\usepackage{multicol}
\usepackage{array}
\usepackage{braket}
\usepackage{epigraph}
\usepackage{rotating, floatflt}
\usepackage{xstring} % Пакет для надежной работы со строками

\usepackage[
backend=biber,
style=numeric-comp,
sorting=none,
giveninits=true
]{biblatex}
\usepackage{csquotes}

\usepackage{xcolor}

\usepackage{ulem}

\usepackage{tikz}
\usetikzlibrary{positioning, calc, intersections, through, arrows, matrix, chains, math}
\usetikzlibrary{arrows.meta}
\usetikzlibrary{shadows}
\usetikzlibrary{graphs}
\usetikzlibrary{graphs.standard}
\usetikzlibrary{patterns}
\usetikzlibrary{decorations.pathreplacing,calligraphy,backgrounds}
\DeclareFieldFormat{usera}{\href{https://arxiv.org/abs/#1}{\texttt{arXiv:#1}}}

\renewbibmacro*{finentry}{\printfield{usera}\newunit\finentry}

\renewbibmacro{in:}{%
  \iffieldundef{usera}
    {\printtext{\bibstring{in}\intitlepunct}}
    {}%
}

\newcommand\norm[1]{\ensuremath{\left\lVert#1\right\rVert}}
\newcommand\abs[1]{\ensuremath{\left\lvert#1\right\rvert}}

\newcommand{\Hcal}{\mathcal{H}}

\newcommand{\R}{\ensuremath{\mathbb{R}}}
\newcommand{\Z}{\ensuremath{\mathbb{Z}}}

\renewcommand{\geq}{\geqslant}

\renewcommand{\leq}{\leqslant}

\newcounter{mcnt}

\newcounter{wordcnt}

\begin{document}

\title{Kolmogorov--Arnold against bounded translations}

\author{Sviatoslav V. Dzhenzher}

\begin{abstract}
    Historically originating from Hilbert's 13th problem, the Kolmogorov--Arnold representation theorem (KART) has recently experienced a major revitalisation through its applications to neural networks, specifically Kolmogorov–Arnold Networks (KANs).
    While the exact representation is well established, its stability under continuous adversarial perturbations of the hidden layer remains a critical open question.
    In this paper, we investigate the robustness of KART against bounded adversarial translations. We provide an explicit, self-contained, and constructive proof of an approximate representation using fixed, piecewise linear ``inner'' functions.
    Crucially, our construction employs a single ``outer'' function that remains invariant for all summands and is independent of the specific adversarial translation, provided its maximum bound is known a priori.
\end{abstract}

\thanks{\hspace{-4mm}
S.\,V. Dzhenzher: sdjenjer@yandex.ru. orcid: 0009-0008-3513-4312
\\
% V.\,Zh. Sakbaev: fumi2003@mail.ru. orcid: 0000-0001-8349-1738
% \\
% V.\,Zh. Sakbaev: Keldysh Institute of Applied Mathematics of Russian Academy of Sciences 125047, Miusskaya pl. 4, Moscow, Russia
% \\
% All authors:
Moscow Institute of Physics and Technology 141701, Institutskii lane, 9, Dolgoprudny, Russia}

\maketitle
\thispagestyle{empty}


\vspace{3mm}

\section{History and introduction}

Historically, this field of research originates from Hilbert's 13th problem \cite{wiki-Hilb-13-problem, Morris-2021-hilbert13}.
In 1836, Hamilton showed \cite{Hamilton1836} that any seventh-degree equation can be reduced to an equation of the form
\[
    x^7 + ax^3 + bx^2 + cx + 1 = 0.
\]
For the quadratic equation \(ax^2+bx+c=0\), we all know the school formula that the solution to the equation \(x = x(a,b,c) = \frac{-b\pm \sqrt{b^2-4ac}}{2a}\), considered a continuous function of three variables, can be expressed as the superposition of arithmetic operations.
Analogously, Hilbert asked whether the solution \(x=x(a,b,c)\) of the seventh-degree equation can be expressed as the superposition of functions of two variables.
For continuous functions, the problem was resolved affirmatively in 1956--57 by Kolmogorov and Arnold in their companion papers \cite{Kolmogorov56, Arnold57, Kolmogorov57}.
Nowadays, their result is known as the Kolmogorov--Arnold representation theorem (KART; see Theorem~\ref{t:kar}; in computer science, it is sometimes referred to as the Kolmogorov Superposition Theorem).

Initially, KART faced long-standing criticism for its non-constructive nature and the low smoothness of the resulting inner and outer functions.
Specifically, Vitushkin \cite{Vitushkin-1954, Vitushkin-overview} and Fridman \cite{Fridman} demonstrated strict limitations regarding the differentiability of these representations.
It might seem amusing that after half a century, even Ismailov's version \cite{Ismailov-2026} for discontinuous functions has found practical applications.
Another line of research was strictly topological, where Ostrand \cite{Ostrand65} generalised KART to functions on finite-dimensional compact spaces, and Sternfeld \cite{Sternfeld85, Sternfeld89} proved the lower bound on the dimension of these compact sets (see also \cite{Dzhenzher24} for an overview).
Nearly three decades passed before it was recognised \cite{HechtNielsen87, Kurkova91, Maiorov-Pinkus} that KART also provides a significant contribution to the topic of neural networks (NNs).
The approach was completely revitalised in 2024 by Kolmogorov--Arnold Networks (KANs) \cite{KAN, KAN2, CKAN}.

In light of the many practical implications of KART, the question of the stability of the Kolmogorov--Arnold representation was posed in \cite{DzhenzherFreedman-25}.
More specifically, the sensitivity of KART under continuous reparameterisations of the hidden layer was investigated.
Unlike coding theory, where the action of the noise is unknown, it was assumed that the action of the reparameterisation is known and that the output layer can be adjusted.
A useful analogy is that of a known adversary attempting to perturb the hidden layer, and we try to withstand this attack by varying the output layer.

In \cite{DzhenzherFreedman-25}, the stability of KART under countable families of adversarial homeomorphisms was established, and several natural questions were posed.
In this paper, we address some of them.

The paper is organised as follows.
Section~\ref{s:defs} presents the necessary background and the main result (Theorem~\ref{t:kart-appx}).
Section~\ref{s:proof} provides the proof of the main result.
Section~\ref{s:discuss} is devoted to the discussion of results and their place among other results on the topic.
Section~\ref{s:concl} concludes the paper and outlines further research.

\section{Background and main results}\label{s:defs}

We say that a continuous tuple \(\psi\colon[\,0,1\,]\to[\,0,1\,]^k\) is a tuple consisting of continuous functions \(\psi_1,\ldots,\psi_k\colon[\,0,1\,]\to[\,0,1\,]\).
Analogously, we say that a continuous matrix \(\phi\colon[\,0,1\,]\to[\,0,1\,]^{m\times k}\) is a tuple consisting of continuous tuples \(\phi_1,\ldots,\phi_m\colon[\,0,1\,]\to[\,0,1\,]^k\); in other words, this is a matrix consisting of continuous functions \(\phi_{i,j}\colon[\,0,1\,]\to[\,0,1\,]\).
We will use the analogous terminology for tuples and matrices with other continuity properties and ranges.

\begin{theorem}[Kolmogorov, Arnold; 1956--57]\label{t:kar}
    Let $n>1$ be an integer. \\
    There exists a fixed ``inner'' continuous matrix \(\phi\colon [\,0,1\,]\to[\,0,1\,]^{n\times2n+1}\) such that \\
    for any ``target'' continuous function \(f\colon [\,0,1\,]^n\to\R\) \\
    there exist ``outer'' uniformly continuous functions \(g_1,\ldots,g_{2n+1}\colon\R\to\R\) such that
    \[
        f(x) = \sum_{i=1}^{2n+1} g_i\left(\sum_{j=1}^n\phi_{j,i}(x_j)\right)
        \quad\text{for any}\quad x=(x_1,\ldots,x_n)\in [\,0,1\,]^n.
    \]
\end{theorem}

In \cite{Lorentz-Golitschek-Makovoz}, the exposition is presented in which the inner functions are chosen of the form \(\phi_{j,i}(x_j) = \lambda_j\phi_i(x_j)\) for \(\lambda_1,\ldots,\lambda_n\) being any rationally independent numbers; moreover, the outer functions \(g_i\) can be chosen equal.
In \cite{Ismailov-2026}, it is shown that the theorem holds for discontinuous functions $f$ if we allow the ``outer'' functions $g_i$ to be discontinuous.
See other versions of KART with different ``inner'' and ``outer'' functions in \cite{Hedberg-appx, Brattka2007, Sprecher1965}.

The following is the version from \cite{DzhenzherFreedman-25} on the stability of KART under ``adversarial'' attacks on the ``inner'' layer.

\begin{theorem}[Dzhenzher, Freedman; 2025]\label{t:kart-adv}
    Let \(n>1\) be an integer.
    Let $\Hcal$ be a countable set of homeomorphisms \(\R^{2n+1}\to\R^{2n+1}\). \\
    There exists a fixed ``inner'' continuous matrix \(\phi\colon [\,0,1\,]\to[\,0,1\,]^{n\times 2n+1}\) such that \\
    for any ``target'' continuous function \(f\colon[\,0,1\,]^n\to\R\) and for any ``adversarial'' \(h=(h_1,\ldots,h_n)\in\Hcal^n\), \\
    there exists a uniformly continuous function \(g\colon\R\to\R\) such that
    \[
        f(x) = \sum_{i=1}^{2n+1} g\left(\sum_{j=1}^n h_{j,i}\bigl(\phi_j(x_j)\bigr)\right)
        \quad\text{for}\quad x=(x_1,\ldots,x_n)\in[\,0,1\,]^n.
    \]
\end{theorem}

For example, if we take all \(h_j\) to be identities, then this result is a reformulation of Theorem~\ref{t:kar}.
An analogous result was proved for the ``inner'' functions of the form \(\phi_{j,i}(x_j) = \lambda_j\phi_i(x_j)\), with the ``adversarial'' homeomorphism acting on \(\phi\) but not on \(\lambda\).

In \cite{DzhenzherFreedman-25}, difficulties regarding equicontinuity arose, preventing similar results from being established for continuous groups of homeomorphisms.
It was mentioned that the approximation trivially works for adversaries from the group of translations \(\R^{2n+1}\to\R^{2n+1}\) if we are allowed to use different ``outer'' functions~$g_i$.
The following theorem partially solves the case for the unique ``outer'' function~$g$.

Hereafter, \(\abs{\cdot}_\infty\) is the max-metric in $\R^m$, and \(\norm{\cdot}\) is the sup-norm for continuous functions.
Recall that the modulus of continuity \(\omega_f\colon[\,0,+\infty)\to[\,0,+\infty)\) of the continuous function \(f\colon[\,0,1\,]^n\to\R\) is
\[
    \omega_f(d) := \sup\bigl\{\abs{f(x)-f(y)} \;:\; \abs{x-y}_\infty \leq d\bigr\}.
\]

\begin{theorem}\label{t:kart-appx}
    Let \(n>1\) and \(m\geq n+1\) be integers.
    Let $C > 0$, $\delta_0>0$ and $0<\gamma<\frac{1}{m}$ be real numbers.
    Then \\ there exists an ``inner'' piecewise linear matrix \(\phi\colon[\,0,1\,]\to\R^{n\times m}\) such that \\
    for any ``target'' continuous function \(f\colon[\,0,1\,]^n\to\R\) \\ there exists an ``outer'' piecewise linear \(\frac{2\norm{f}}{\delta_0(m-n)}\)-Lipschitz function $g\colon\R\to\R$ such that \(\norm{g} \leq \frac{1}{m-n}\norm{f}\) and \\
    for any ``adversarial'' \(t\in[\,-C,C\,]^m\),
    \[
        \abs{f(x) - \sum_{i=1}^mg\left(\sum_{j=1}^n\phi_{j,i}(x_j)+t_i\right)} \leq
        \omega_f(\gamma(m-1)) + \frac{n}{m-n}\norm{f}
        \quad\text{for}\quad x=(x_1,\ldots,x_n)\in[\,0,1\,]^n.
    \]
\end{theorem}

\section{Proof of Theorem~\ref{t:kart-appx}}\label{s:proof}

This entire section is devoted to the proof of Theorem~\ref{t:kart-appx}.
The section is structured so that necessary constructions come first, followed by proofs that these constructions satisfy the required properties.
Some technical facts are collected in the four lemmas below.

Denote \([k] := \{1,\ldots,k\}\) for an integer $k$.

In the following three paragraphs, we construct the ``inner'' matrix \(\phi\colon[\,0,1\,]\to\R^{n\times m}\).

The \emph{red segments} of rank $i\in[m]$ are
\[
    [\,0,1\,]\cap \left[\,\gamma(mk+i-(m-1)),\gamma(mk+i)\,\right]
    \quad\text{for}\quad k\in\Z;
\]
these are mostly closed segments of length \(\gamma(m-1)\) with possibly shorter segments or points near \(\{0,1\}\).

Let $K\approx\lceil\frac{1}{\gamma m}\rceil$ be the maximal number of non-empty red segments of the same rank.
Denote
\[
    \delta := 2C+\delta_0
    \quad\text{and}\quad
    \Delta := \delta K^n;
\]
these represent the step sizes in the construction of $\phi$.

Now, let red segments of rank $i$ be numbered like \(S_{i,0}, S_{i,1}, \ldots, S_{i,K_i}\) in order from left to right, \(K_i < K\).
For \(i\in[m]\), \(j\in[n]\), construct piecewise linear functions \(\phi_{j,i}\) as follows: make them constant on the red segments \(S_{i,k}\) and equal to
\[
    \varphi_{j,i,k} := i\Delta + \delta K^{j-1}k =
    \delta (iK^n + kK^{j-1}),
\]
and extend them linearly from these ``forced'' values between the red segments.

Now, fix any continuous function \(f\colon[\,0,1\,]^n\to\R\).
We need to construct the ``outer'' function $g$.
A \emph{red $n$-solid} of rank $i\in[m]$ is the Cartesian product of $n$ red segments of rank~$i$.
For each red $n$-solid~\(S\), choose a point \(p(S)\in S\).
For an $n$-solid \(S=S_{i,k_1}\times \ldots \times S_{i,k_n}\), denote
\[
    \Phi(S) := \sum_{j=1}^n \varphi_{j,i,k_j} \in [\,in\Delta, in\Delta + \Delta-\delta\,] =: W_i.
\]
Finally, define the piecewise linear function \(g\colon\R\to\R\) by the ``forced'' values
\[
    g(y) := \frac{f(p(S))}{m-n},
    \quad y \in [\,\Phi(S)-C,\Phi(S)+C\,].
\]
The definition is meaningful since the segments \([\,\Phi(S)-C,\Phi(S)+C\,]\) do not intersect due to the following two lemmas.

\begin{lemma}
    For two different $n$-solids \(S,S'\) of the same rank,
    \[
        \abs{\Phi(S)-\Phi(S')} \geq \delta.
    \]
\end{lemma}

\begin{proof}
    This follows since \(\Phi(S)\) is, up to a factor of $\delta$, a number in the numeral system with base $K$.
\end{proof}

\begin{lemma}
    For two different $n$-solids \(S,S'\) of different ranks,
    \[
        \abs{\Phi(S)-\Phi(S')} \geq (n-1)\Delta + \delta.
    \]
\end{lemma}

\begin{proof}
    This follows since the lower bound is the distance between neighbouring windows \(W_i,W_{i+1}\).
\end{proof}

The fact that $g$ is \(\frac{2\norm{f}}{\delta_0(m-n)}\)-Lipschitz follows, since, by the previous lemmas, the distances between different intervals \([\,\Phi(S)-C, \Phi(S)+C\,]\) are at least \(\delta-2C = \delta_0\).

Take any \(t\in[\,-C,C\,]^m\).
Fix an arbitrary \(x \in [\,0,1\,]^n\).
It remains to prove the final upper bound:
\[
    \abs{f(x) - \sum_{i=1}^mg\left(\sum_{j=1}^n\phi_{j,i}(x_j)+t_i\right)} \leq
    \omega_f(\gamma(m-1)) + \frac{n}{m-n}\norm{f}.
\]

\begin{lemma}\label{l:x-in-solids}
    Any point \(x\in[\,0,1\,]^n\) lies in red $n$-solids of at least $m-n$ different ranks.
\end{lemma}

\begin{proof}
    The argument is well known; for example, see~\cite{Brattka2007}. It is sufficient to prove that any \(s\in[\,0,1\,]\) is placed in red segments of at least $m-1$ different ranks.
    This holds since
    \[
        s \in [\,0,1\,]\cap \left[\,\gamma(mk+i-(m-1)),\gamma(mk+i)\,\right] \Longleftrightarrow
        \frac{s}{\gamma} \in [\,mk+i-(m-1), mk+i\,],
    \]
    and taking $k$ such that \(\frac{s}{\gamma} \in [\,mk-(m-1), mk+1)\), we obtain that the inclusion
    \[
        \frac{s}{\gamma} \in [\,mk+i-(m-1), mk+i\,]
    \]
    holds for at least $m-1$ different ranks~$i$.
\end{proof}

By Lemma~\ref{l:x-in-solids}, for the given $x$ we are able to take $m-n$ red $n$-solids \(S_i \ni x\) of different ranks \(i\in I\), where the size of $I$ is exactly \(m-n\).
Then
\[
    \abs{f(x) - \sum_{i=1}^mg\left(\sum_{j=1}^n\phi_{j,i}(x_j)+t_i\right)} \leq
    \sum_{i\in I}\abs{\frac{f(x)}{m-n} - g\!\left(\sum_{j=1}^n\phi_{j,i}(x_j)+t_i\right)} + \sum_{i\notin I}\abs{g\!\left(\sum_{j=1}^n\phi_{j,i}(x_j)+t_i\right)},
\]
and the following two lemmas conclude the proof.

\begin{lemma}
    \[
        \sum_{i\in I}\abs{\frac{f(x)}{m-n} - g\!\left(\sum_{j=1}^n\phi_{j,i}(x_j)+t_i\right)}
        \leq \omega_f(\gamma(m-1)).
    \]
\end{lemma}

\begin{proof}
    Fix any \(i\in I\).
    Then
    \[
        \sum_{j=1}^n\phi_{j,i}(x_j) = \Phi(S_i).
    \]
    Since \(\abs{t}_\infty \leq C\), and since \(g\) is constant on the intervals \([\,\Phi(S)-C, \Phi(S)+C\,]\),
    \[
        g\left(\sum_{j=1}^n\phi_{j,i}(x_j) + t_i\right) = g\Bigl(\Phi(S_i)\Bigr) = \frac{f(p(S_i))}{m-n}.
    \]
    Since \(\abs{x-p(S_i)}_\infty \leq \gamma(m-1)\),
    \[
        \abs{f(x)-f(p)} \leq \omega_f(\gamma(m-1)),
    \]
    and thus each summand in the statement does not exceed \(\frac{\omega_f(\gamma(m-1))}{m-n}\).
\end{proof}

\begin{lemma}
    \[
        \sum_{i\notin I}\abs{g\!\left(\sum_{j=1}^n\phi_{j,i}(x_j)+t_i\right)} \leq
        \frac{n}{m-n}\norm{f}.
    \]
\end{lemma}

\begin{proof}
    It is clear since each summand does not exceed \(\frac{\norm{f}}{m-n}\).
\end{proof}

\section{Discussion of the result}\label{s:discuss}

First, note that the idea of the proof relies on the combinatorial argument about red $n$-solids covering points, but it uses this argument slightly differently.
In the classical expositions \cite{Hedberg-appx, Lorentz-Golitschek-Makovoz}, the ``inner'' functions were constructed via rationally independent numbers (as linear rational combinations of square roots of prime numbers, to be precise).
In our proof, we designed the ``inner'' functions in such a way that their values on the red segments are located far enough apart.
It is worth mentioning that in the original KART and in \cite{DzhenzherFreedman-25}, the ``inner'' matrices $\phi$ took values in the unit cube, while in Theorem~\ref{t:kart-appx}, the ``inner'' functions take values in the entire $\R$.
Furthermore, as the perturbation $C$ increases, the wider the range of the ``inner'' matrix becomes.
Consequently, it could be argued that the problem posed in \cite{DzhenzherFreedman-25} is addressed at the expense of a minor trade-off.

The problem with the result being only an approximate, rather than an exact, outcome is that we cannot control the modulus of continuity of the function $f$.
If we knew that $f$ is $L$-Lipschitz, then we would be able to take $\gamma$ small enough so that
\[
    \omega_f(\gamma(m-1)) \leq \gamma(m-1)L < \varepsilon,
\]
and we would obtain the approximation with an error of at most
\[
    \varepsilon + \frac{n}{m-n}\norm{f};
\]
for $m=2n+1$, it is exactly like \cite[inequalities~(1,2)]{DzhenzherFreedman-25}.
However, since the space of Lipschitz functions with bounded Lipschitz constants is not separable, we cannot approximate a function $f$ with a smooth enough Lipschitz function and fuel the recursive approximation, as is customary in classical expositions (for example, see~\cite[statements and applications of lemmas~3.2 and~4.2]{DzhenzherFreedman-25}).
For the same reasons, we cannot prove the openness and the denseness of the appropriate matrices $\phi$ in order to apply the Baire category argument (which is originally due to Kahane \cite{Kahane}).

Note also that, despite the fact that Theorem~\ref{t:kart-appx} is formulated for $m\geq n+1$, it works well only for \(m\geq 2n+1\).
Indeed, for \(n+1\leq m \leq 2n\), the error \(\frac{n}{m-n}\norm{f}\) can be defeated by the error \(\norm{f}\) using the function \(g = 0\) without any complex preparations.
The choice of $m$ involves a trade-off between two competing factors: on the one hand, with the growth of $m$, the error \(\omega_f(\gamma(m-1))\) increases; on the other hand, the error \(\frac{n}{m-n}\norm{f}\) decreases.

Finally, note that in \cite{DzhenzherFreedman-25}, the ``outer'' function depends on the choice of the ``adversarial'' reparameterisation, while in Theorem~\ref{t:kart-appx}, it does not.
The cost of this independence is that we must know the maximal scope of the ``adversarial'' translation even before we construct the ``inner'' functions.

\section{Conclusion and further research}\label{s:concl}

In this paper, we have investigated the stability of the Kolmogorov--Arnold representation under ``adversarial'' reparameterisations of the hidden layer.
The main result (Theorem~\ref{t:kart-appx}) gives an explicit, piecewise linear construction of an inner matrix $\phi\colon [\,0,1\,] \to \R^{n \times m}$ that satisfies two important additional properties: the ``outer'' function $g$ is taken to be the \emph{same} for all $m$ summands, and the construction (in contrast to Theorem~\ref{t:kart-adv}) does not depend on the particular ``adversarial'' translation $t \in [\,-C, C\,]^m$, provided that the bound $C$ is known in advance.
The latter fact is the price we pay: the ``inner'' functions $\phi_{j,i}$ have to take values in the whole real line, and the larger the allowed scope of the adversary, the wider the range of $\phi$.

The proof of Theorem~\ref{t:kart-appx} is constructive and self-contained.
It is interesting to compare this proof with the similar ideas from \cite{Sprecher96, Sprecher97, Braun-Griebel-2009}, where the first constructive proof of the KART was given.

In the course of the proof, we identified two important structural limitations of the result. First, the upper bound on the translation has to
be known \emph{before} the ``inner'' function is constructed, which makes the result ``non-adaptive'' with respect to the reparameterisation.
Second, the loss of control of the modulus of continuity of $f$ prevents us from obtaining a \emph{sharp} result.
Whether either of these obstructions can be overcome by a more sophisticated choice of $\phi$ is, to the best of our knowledge, an open question.

We believe that the techniques introduced in this paper can serve as a starting point for a finer study of the stability of KART under larger
classes of adversarial actions, and we hope to return to these questions in future work.

\printbibliography
% % \printbibitembibliography

\end{document}